\documentclass[letterpaper, 10 pt, conference]{ieeeconf}
\IEEEoverridecommandlockouts                              
\usepackage{paralist,subcaption}
\usepackage[pdftex]{graphicx}
\usepackage{graphicx,amsmath,amsfonts,comment,algorithmic,algorithm,amssymb, color, framed, soul}
\usepackage{pgfplots} 
\pgfplotsset{compat = newest}
\usepackage{tikz}
\usepackage{pgfgantt}
\usetikzlibrary{shapes,arrows,positioning,matrix,chains,fit}
\usetikzlibrary{backgrounds}

\newtheorem{theorem}{Theorem}

\title{\LARGE \bf Reinforcement Learning for Task Specifications with Action-Constraints}

\author{Arun Raman$^{1}$, Keerthan Shagrithaya$^{1}$, Shalabh Bhatnagar$^{1}$
\thanks{*The work of the first author was supported by the C. V. Raman postdoctoral fellowship. The work of the second and third authors was supported by the J. C. Bose fellowship of the third author.}
\thanks{$^{1}$ Department of Computer Science and Automation, Indian Institute of Science, Bangalore, India
        {\tt\small\{arunraman, keerthans, shalabh\}@iisc.ac.in}}%
}
    
\begin{document}

\maketitle
\begin{abstract}

In this paper, we use concepts from supervisory control theory of discrete event systems to propose a method to learn optimal control policies for a finite-state Markov Decision Process (MDP) in which (only) certain sequences of actions are deemed unsafe (respectively safe). We assume that the set of action sequences that are deemed unsafe and/or safe are given in terms of a finite-state automaton; and propose a supervisor that disables a subset of actions at every state of the MDP so that the constraints on action sequence are satisfied. Then we present a version of the $Q$-learning algorithm for learning optimal policies in the presence of non-Markovian action-sequence and state constraints, where we use the development of reward machines to handle the state constraints. We illustrate the method using an example that captures the utility of automata-based methods for non-Markovian state and action specifications for reinforcement learning and show the results of simulations in this setting. 

\end{abstract}

\section{Introduction}\label{intro}
In reinforcement learning (RL) \cite{sutton2018reinforcement}, an agent evaluates its optimal behaviour from experience gained through interaction with the environment which is guided by the rewards that an action at a given state fetches from it. Safe reinforcement learning is a paradigm wherein the idea is to not violate a set of constraints during the process of learning \cite{garcia2015comprehensive}. In many applications, the constraints may be given in terms of the sequences of actions. For example, consider an agent, that can move in the four cardinal directions, traversing in a grid-like warehouse environment. In this case, the safety requirement corresponding to one-way traffic can have a specification of not taking two consecutive right turns which will result in a U-turn. 
In several scheduling applications, if scheduling an item is interpreted as an action, then the constraints on the sequence of actions arise naturally. For example, in real-time scheduling of home appliances, the dryer cannot be scheduled before the washer \cite{sou2011scheduling, kaur2021machine}. In this paper, we are interested in a method for obtaining optimal control policies for such action-constrained scenarios.

Oftentimes, the system states and constraints are affected by uncontrollable events. For example, some actuators of a robot can break down due to faults. The task specification for an agent should be robust to such uncontrollable events. The subject of uncontrollable events altering the state of the system is well studied in the theory of supervisory control of discrete event systems\cite{ramadge1989control, cassandras2009introduction}. A {\em Discrete Event System} (DES) is a discrete-state system, where the state changes at discrete time instants due to the occurrence of events. If we associate a (not necessarily unique) label to each event of the DES, then its behaviour can be described by the language it generates using the alphabet defined by the set of labels. In such a setting, the desired behaviour of the DES is also specified by a language specification. A \textit{supervisory policy} enforces a desired language specification by disabling a subset of events known as the \textit{controllable events}. The complementary set of events that cannot be disabled by the supervisor are called \textit{uncontrollable events}. It is important to note that the supervisory policy cannot force an event to occur, it can only disable a controllable event. As such, a supervisory policy aims for minimum intervention in the operation of the system--- only when a particular event can lead to the violation of a specification. In this paper, we interpret the uncontrollable events as (uncontrollable) actions which can then be used to specify an overall action-constraint for the system by taking the possibilities of uncontrollable events into account. We limit our analysis to those action-constraints that can be represented by a finite automaton.

The environment in the Reinforcement Learning (RL) paradigm is typically modeled as a Markov Decision Process (MDP). If all the events of a Disrete Event System (DES) are controllable, then there are parallels between supervisory control of a DES and the theory of synthesizing policies for an MDP. In particular, fundamentally, both cases involve solving a sequential decision-making problem in which the state evolution in the model is Markovian in nature. On the other hand, a key difference between the two is that the action space of an MDP does not, generally, consider actions to be uncontrollable. This is because even if such aspects do exist in the system, they can be absorbed in the MDP model by introducing a probability distribution for the next state such that it accounts for the possibility of occurrence of an uncontrollable action. Such an approach works because the objective in the MDP setting is to find an optimal policy that minimizes an expected cost. However, if uncontrollable actions appear in the safety specification, then they cannot be abstracted by a probability distribution anymore and must appear explicitly in the MDP model. In such scenarios, the `safety' part of the safe RL paradigm naturally invites the results from supervisory control theory. Our main contributions in this paper are the following:
\begin{enumerate}
    \item To the best of our information, this paper is the first work that bridges the areas of supervisory control theory (SCT) and reinforcement learning. Such an association permits us to directly use an important theoretical result from SCT: to determine if there exists a policy that satisfies the action-constraints in the presence of uncontrollable actions; offline, before training. If there does not exist such a policy, then there is another important result in SCT that, informally, evaluates the closest to the given constraints that a supervisory policy can enforce. These results are relevant in safe reinforcement learning for complex systems in which correctly formulating the constraints is a task in itself.  
    \item We present a version of Q-learning algorithm to learn optimal control policies for an underlying finite-state MDP in which:
    \begin{enumerate}
        \item a subset of actions might be uncontrollable, and
        \item the task specification or the safety constraints for the system is given in terms of sequences of actions modeled by a finite state automaton. The constraints can be given in terms of sequences that are safe and/or unsafe. This generality simplifies the problem formulation to some extent. 
    \end{enumerate}
    The adapted Q-learning algorithm also integrates the work on reward machines\cite{icarte2018using} that supports reward function specification as a function of state.
    \item Experimental results for an illustrative example that captures the utility of automata-based methods for non-Markovian state and action specifications for reinforcement learning are presented.
\end{enumerate}

The rest of the paper is organized as follows. Section \ref{model} presents some notations and definitions that we will be using in the rest of the paper. It also presents the relevant results from supervisory control theory. Section \ref{aa} discusses supervisor synthesis. We will largely be working with finite automata in this section and use the development for supervisor synthesis to propose a version of $Q$-learning algorithm \cite{watkins1992q} that can be used to learn an optimal policy in the presence of state and action constraints. In Section \ref{example}, we present an example to illustrate the use of automata-based methods to handle non-Markovian state and action-constraints. We use the proposed method for action-constraints and reward machines\cite{icarte2018using} for the state constraints. 
We show the results of the experiments and observe that the Q-learning algorithm adapted to this setting helps to make the agent learn to complete the task optimally 100\% of the time. We conclude the paper with Section \ref{concl}.

\section{Motivation, Notations and Definitions}\label{model}
In a reinforcement learning problem, the environment is modeled as a Markov Decision Process $\mathcal{M} = (S, A, r, p, \gamma)$ where $S$ and $A$ denote the finite set of states and actions, respectively, $r: S \times A \times S \rightarrow \mathbb{R}$ denotes the reward function, $p(s_{t+1}|s_t, a_t)$ is the transition probability distribution and $\gamma \in (0,1)$ is the discount factor. In the paradigm of safe reinforcement learning, the objective is to learn optimal policies while satisfying a set of constraints. In this paper, we consider safety constraints that are given as a set of sequence of actions of the MDP. Next, we discuss the example gridworld in Fig. \ref{fig:my_label1} to motivate the development. 

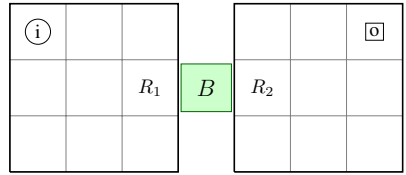
\begin{figure}[htbp]
    \centering
    \resizebox{.3\textwidth}{!}{
    \begin{tikzpicture}
    \draw[gray, very thin] (0, 0) grid (3, 3);
    \draw[gray, very thin] (4, 0) grid (7, 3);
    \draw [thick] (0,0) -- (3,0) -- (3,3) -- (0,3) -- (0,0);
    \draw [thick] (4,0) -- (7,0) -- (7,3) -- (4,3) -- (4,0);
    \node[shape=circle,draw,inner sep=2pt] at (0.5, 2.5) () {i};
    \node[shape=rectangle,inner sep=2pt] at (2.5, 1.5) () {$R_1$};
    
    \node[shape=rectangle,draw,inner sep=2pt] at (6.5, 2.5) () { o};
    \node[shape=rectangle,inner sep=2pt] at (4.5, 1.5) () {$R_2$};
    \node[shape=rectangle,fill=green!20!white, draw=green!40!black,inner sep=8pt] at (3.5, 1.5) () {\large $B$};
    \end{tikzpicture}}
    \caption{A pick-up and delivery grid world}
    \label{fig:my_label1}
\end{figure}
Consider a manufacturing facility layout as shown in Fig. \ref{fig:my_label1}. When a work piece appears at location \textcircled{i}, the robot $R_1$ picks it up and delivers it to the buffer zone $B$. The robot $R_2$ picks the item from $B$ and delivers it to the processing facility \fbox{o}. The buffer $B$ can only store one item at a time; so it has two states Empty($E$) and Full($F$). The robot state is determined by the position that it occupies in the grid and by a binary variable indicating whether it has picked up an item or not. There are six controllable actions for each robot where four of them correspond to their movement in the grid and two correspond to picking and dropping of an item ($p_i, d_i$), $i \in \{1, 2\}$. Additionally, we assume that robot $R_2$ can either be Working(W) or Down(D). There are two uncontrollable actions $\{\phi, \rho\}$ denoting the fault and rectification actions which take $R_2$ from $W$ to $D$ and vice versa respectively. We call them uncontrollable because they are caused by external factors/agents not modeled in this set up. We can also assume that an uncontrollable action makes an item appear in \textcircled{i}. The objective is to design an optimal policy to satisfy the specifications below:
\begin{enumerate}
    \item \label{w} $R_1$ delivers an item to $B$ only if $B$ is in $E$.
    \item $R_2$ picks an item from $B$ only if $B$ is in $F$ (thereby driving $B$ to $E$).
    \item \label{x} $R_1$ and $R_2$ deliver a picked item in no more than $5$ steps.
\end{enumerate}
The above specifications `look' okay but are not robust to uncontrollable actions. To see this, consider a string of actions $\sigma$. Formally, the specification in item \ref{x} says that for a substring $p_1\sigma d_1$, it must be that $|\sigma| \leq 5$ where $|\sigma|$ denotes the size of $\sigma$, and $p_1$ and $d_1$ denote picking and dropping of item by $R_1$ respectively. Consider the scenario in which $R_1$ takes an action $p_1$ when the buffer is full, with the assumption that it will be emptied by $R_2$ in the next step. Such an action is consistent with item \ref{w} of the specification. At this point, if the uncontrollable action $\phi$ occurs, then the satisfaction of item \ref{x} of the specification will entirely depend on the occurrence of the uncontrollable rectification action $\rho$. As such, the specifications above are not robust to uncontrollable actions. That is, the occurrence of an uncontrollable action resulted in the generation of a string of actions that was not in the specification. In what follows, we discuss developments from supervisory control theory which facilitate a formal analysis of such aspects.

An automaton is a 5-tuple $G=(Q, \Sigma, \delta, q_0, Q_m)$\cite{cassandras2009introduction}. Here $Q$ and $\Sigma$ denote the set of \textit{states} and (labelled) \textit{actions} respectively. The \textit{initial state} is $q_0 \in Q$, and $Q_m \subseteq Q$ is called the set of \textit{marked states}. The states can be marked for different reasons and we discuss more about this later in the paper. The function $\delta: \Sigma \times Q \rightarrow Q$ is the \textit{state transition function}, where for $q \in Q, \sigma \in \Sigma, \delta(\sigma, q) = \widehat{q}$, implies there is an action labeled $\sigma$ from state $q$ to state $\widehat{q}$. In general, the state transition function $\delta$ may not be defined for all state-action pairs. The \textit{active action function} $\Gamma: Q \rightarrow 2^{\Sigma}$, identifies the set of all actions $\sigma \in \Sigma$ for which $\delta(\sigma, q)$ is defined. 

Let $\Sigma^*$ denote the set of all finite strings of elements of $\Sigma$ including the empty string $\epsilon$. We write $\delta^*(\omega, q) = \widehat{q}$ if there is a string of actions $\omega \in \Sigma^*$ from $q$ to state $\widehat{q}$, that is if $\widehat{q}$ is reachable from $q$. Given $G$, the set of all admissible strings of actions is denoted by $L(G)$ and is referred to as the \textit{language} generated by $G$ over the alphabet $\Sigma$. Formally, for the initial state $q_0$:
\begin{align}\label{xx}
    L(G) = \{\omega \in \Sigma^*: \delta^*(\omega, q_0) \text{ is defined} \}.
\end{align}
The marked language, $L_m(G) \subseteq L(G)$, generated by $G$ is 
\begin{align}
    L_m(G) = \{\omega \in \Sigma^*: \delta^*(\omega, q_0) \in Q_m \}. \nonumber
\end{align}
A string $u$ is a {\it prefix} of a string $v \in \Sigma^*$ if for some $w \in \Sigma^*$, $v = uw$. If $v$ is an admissible string in $G$, then so are all its prefixes. We define the {\it prefix closure} of $L \subseteq \Sigma^*$ to be the language $\overline{L}$ defined as:
\begin{align}
    \overline{L} = \{ u: uv \in L \text{ for some } v \in \Sigma^* \}. \nonumber
\end{align}


We can use the above notations to interpret the MDP $\mathcal{M}$ as a language generator: $(S, A, \Sigma, \mathcal{L}, r, p, \gamma)$ where $\Sigma$ denotes a set of labels and $\mathcal{L}: A \rightarrow \Sigma$ denotes a labelling function that associates a label to each action. For simplicity, we assume that each action is uniquely labeled and, with some abuse of notation, use $\Sigma$ as a proxy for the set of actions. Then the function $\delta: \Sigma \times S \rightarrow S$ is the \textit{state transition function}, where for $q \in S, \sigma \in \Sigma: \delta(\sigma, q) = \widehat{q}$ means that there is an action $a \in A$ labeled $\sigma$ such that $p(\widehat{q}|q, a) > 0$. The language generated can be defined appropriately. 

A policy $\pi(a|s)$ in the context of RL is a probability distribution over the set of actions $a \in A$ at a given state $s \in S$. At each time step, $t$, the agent is in a particular state, say $s_t$. It selects an action $a_t$ according to $\pi(\cdot|s_t)$ and moves to a next state $s_{t+1}$ with probability $p(s_{t+1}|s_t, a_t)$. It also receives a reward $r(s_{t+1}, a_t, s_t)$ for this transition. The process then repeats from $s_{t+1}$. The agent's goal is to find a policy $\pi^{*}$ that maximizes the expected discounted future reward from every state in $S$. Next we introduce the notion of supervisory control which trims the set of candidate policies to only those that do not violate the safety specification. 

The set of actions ($\Sigma$) is partitioned into the set of \textit{controllable} actions ($\Sigma_c$) and \textit{uncontrollable} actions ($\Sigma_u$). More specifically, $\Sigma_c$ (resp. $\Sigma_u$) denotes the set of actions that can (resp. cannot) be disabled by the \textit{supervisor}. Formally, a supervisor $\mathcal{S}: L(\mathcal{M}) \rightarrow 2^{\Sigma}$ is a function from the language generated by $\mathcal{M}$ to the power set of $\Sigma$. We use $\mathcal{S}/\mathcal{M}$ to denote the supervised MDP (SMDP) in which supervisor $\mathcal{S}$ is controlling the MDP $\mathcal{M}$. Note that the supervisor only disables an action and it does not choose which action should be taken at any time instant. That (optimal) choice is determined by $\pi(a|s)$. The uncontrollable actions are always enabled by the supervisor.


We are interested in evaluating an optimal policy $\pi^*$ that satisfies the constraints specified in terms of a set of sequence of actions. The first question that is of interest to us is if such a policy exists. We can directly use the {\it controllability condition} from the supervisory control theory literature to answer this question \cite{ramadge1989control, cassandras2009introduction}:
\begin{theorem}\label{ct}
Let $K$ be a desired language specification over the alphabet $\Sigma$ denoting the safety constraints for the MDP $\mathcal{M}$. There is a supervisor $\mathcal{S}$ such that $L(\mathcal{S}/\mathcal{M}) = \overline{K}$ if and only if $\overline{K}\Sigma_u \cap L(\mathcal{M}) \subseteq \overline{K}$.
\end{theorem}
\begin{proof}
See Section 3.4.1 of \cite{cassandras2009introduction}.
\end{proof}
The above condition states that there is a supervisor that enforces the desired specificaiton $K$ if and only if the occurrence of an uncontrollable action in $\mathcal{M}$ from a prefix of $K$ results in a string that is also a prefix of $K$. For the gridworld example in Fig. \ref{fig:my_label1}, the uncontrollable action $\phi$ resulted in a string from which it can generate a string in which item \ref{x} of the specification was not satisfied. Therefore, it violated the controllability condition.

If the condition in Theorem \ref{ct} is not satisfied, then there does not exist a supervisor that can enforce the given specification. In such a case, the best that the supervisor can do is enforce the \emph{supremal controllable sublanguage} of $K$, which is the union of all controllable sublanguages of $K$ \cite{wonham1987supremal}: 
\begin{align}\label{contr}
    K^{\uparrow} = \bigcup_i\ \{K_i: (K_i \subseteq K) \wedge (\overline{K}_i\Sigma_u \cap L(G) \subseteq \overline{K}_i)\}. \nonumber
\end{align}
Without going into the formal analysis, the supremal controllable sublanguage (the modified specification) for the gridworld example would involve $R_1$ not picking an item unless $B$ is empty.

In this paper, we will assume, in different ways, that the given specification can be modeled as a finite automaton and that it is controllable. 
An automaton is represented by a directed graph (state transition diagram) whose nodes correspond to the states of the automaton, with an edge from $q$ to $\widehat{q}$ labeled $\sigma$ for each triple $(\sigma, q, \widehat{q})$ such that $\widehat{q}=\delta(\sigma, q)$. 

We use the concept of reward machines to handle the state constraints\cite{icarte2018using}. Given a set of propositional symbols $\mathcal{P}$, a set of (environment) states $S$, and a set of actions $A$, a \textit{Reward Machine (RM)} is a tuple $\mathcal{R}_{\mathcal{P}SA} = \langle U, u_0, \delta_u, \delta_r \rangle$ where $U$ is a finite set of states, $u_0 \in U$ is an initial state, $\delta_u : U \times 2^{\mathcal{P}} \rightarrow U$ is the state-transition function and $\delta_r :U\times U \rightarrow  [S\times A\times S \rightarrow \mathbb{R}]$ is the reward-transition function. An MDP with a Reward Machine (MDPRM) is a tuple $T = \langle S, A, p, \gamma, \mathcal{P}, L, U, u_0,\delta_u,\delta_r \rangle$, where $S, A, p$, and $\gamma$ are defined as in an MDP, $\mathcal{P}$ is a set of propositional symbols, $L$ is a labelling function $L : S \rightarrow 2^{\mathcal{P}}$, and $U, u_0, \delta_u,$ and $\delta_r$ are defined as in an RM. The operation of an MDPRM is as follows. At every decision epoch of $\mathcal{M}$, there is a transition in $\mathcal{R}_{\mathcal{P}SA}$ also. If the current state of $\mathcal{R}_{\mathcal{P}SA}$ is $u$ and an action $a$ is taken from $s$ in $\mathcal{M}$, then $\delta_u(a, u)$ is the next state of $\mathcal{R}_{\mathcal{P}SA}$ and it outputs a reward function $\delta_r(u, \delta_u(a, u))$ which determines the reward in $\mathcal{M}$ when action $a$ is taken from $s$. The main idea in handling state constraints using RMs is to associate barrier rewards with state sequences that violate the desired specification. Reward machines can be used to model any Markovian reward function. They can also used to express any non-Markovian reward function as long as the state history $S^*$ on which the reward depends can be distinguished by different elements of a finite set of regular expressions over $\mathcal{P}$. 

\section{Reinforcement Learning with Action-Constraints}\label{aa}
Consider an MDP $\mathcal{M} = (S, A, \Sigma, L, r, p, \gamma)$. The action-constraints can be given in terms of the sequence of actions that are safe to perform or by sequences that are unsafe. We propose methods for supervisor synthesis for both cases in this section. 

Consider an automaton $G_s' = (Q_s', \Sigma, \delta_s', q_{0s}, \emptyset)$ that is the model of a sequence of actions that is safe with an empty marked set of states. We construct an automaton $G_s = (Q_s, \Sigma, \delta_s, q_{0s}, \{s_a\})$ from $G_s'$ by appropriately adding a (marked) state $s_a$ to flag the occurrence of an action violating the specification. The formal construction is as follows:
\begin{enumerate}
    \item $G_s$ has one additional state: $Q_s = Q_s' \cup \{s_a\}$, $Q_m=\{s_a\}$.
    \item The transition function for $G_s$ is identical to that of $G_s'$ for the state-action pairs for which $\delta_s'$ is defined.
    $\forall s\in Q_s', \forall \sigma \in \Gamma(s)$: $\delta_s(\sigma, s) = \delta_s'(\sigma, s)$.
    \item \label{item} For the state-action pairs for which $\delta_s'$ is not defined, we have: $\forall s\in Q_s', \forall \sigma \in (\Sigma - \Gamma(s))$: $\delta_s(\sigma, s) = s_a$.
    \item If $G_s$ reaches $s_a$, it stays there: $\forall \sigma \in \Sigma$: $\delta_s(\sigma, s_a) = s_a$.
\end{enumerate}
At every decision epoch of $\mathcal{M}$, there is a transition in $G_s$ also. If the current state of $G_s$ is $s$ and an action $a$ is taken in $\mathcal{M}$, then $\delta_s(a, s)$ is the next state of $G_s$. Now, if an action $\omega$ is taken in $\mathcal{M}$ for which $\delta_s'(\omega, s)$ is not defined, then it means that it violates the safety constraint. If $\delta_s'(\omega, s)$ is not defined for $\omega \in \Sigma$ and $s \in Q_s'$, then as per item \ref{item} in the above construction, the state of $G_s$ will transition to $s_a$. The state $s_a$ of $G_s$ indicates that an action constraint has been violated. A supervisor for this case will simply disable actions that lead to a transition to the state $s_a$. If there is an uncontrollable action that can result in a transition to $s_a$, then it means that the specification does not satisfy the controllability condition of theorem \ref{ct} and there does not exist a policy that can enforce the given specification. Automaton $G_s$ in Fig. \ref{F} is an example in which $d_2d_3p_1p_2p_3$ is the sequence of actions that are to be performed when an uncontrollable action $d_u$ happens. 

Next, we consider the case in which the constraints are given in terms of the sequence of actions that are unsafe\footnote{One way of handling this situation would be to evaluate the sequence of actions that are safe and then proceed with the earlier method (we assume that the set of safe sequences can be represented by an automaton).}. Recall that $L_m(\mathcal{G}) \subseteq L(\mathcal{G})$ denotes the marked language generated by an automaton $\mathcal{G}$. We assume that the set of strings of actions that are unsafe $S_{un}$ are given as an automaton $H_1 = (Q_{h}',\Sigma_h', \delta_h', q_{0h}', Q_{mh'})$ for which $L_m(H_1) = S_{un}$. We then modify $H_1$ to obtain an automaton $H=(Q_{h},\Sigma_h, \delta_h, q_{0h}, Q_{mh})$ as follows: 
\begin{enumerate}
    \item $Q_h = Q_h'$, $\Sigma_h  = \Sigma$,  $q_{0h} = q_{0h}'$, $Q_{mh}  = Q_{mh}'$.
    \item $\forall q \in Q_h'$, $\forall \sigma \in \Gamma_h'(q)$, $\delta_h(\sigma, q)  = \delta_h'(\sigma, q)$.
    \item $\forall q \in Q_h'$, $\forall \sigma \in \Sigma - \Gamma_h'(q)$, $\delta_h(\sigma, q) = q_{0h}$.
\end{enumerate}
That is, $H$ can be constructed from $H_1$ by adding additional outgoing arcs from every state in $H_1$ to the initial state, corresponding to actions in $\Sigma$ that do not already appear as an outgoing arc in $H_1$. This modification ensures that for every action in $\mathcal{M}$, a corresponding transition is defined in $H$. The operation of $H$ is the same as that of $G_s$. At every decision epoch of $\mathcal{M}$, there is a transition in $H$. If the current state of $H$ is $s$ and an action $a$ is taken in $\mathcal{M}$, then $\delta_H(a, s)$ is the next state of $H$, where $\delta_H$ is the transition function of $H$. 
For simplicity, we label every $q \in Q_{mh}$ by $s_a$. The supervisor will disable every controllable action which results in a next state that is marked ($s_a$). Suppose $\mathcal{M}$ has generated a string $\sigma$ that has not resulted in a transition to $s_a$ in $H$, then it means that $\sigma$ is a prefix of an admissible string. Following $\sigma$, if there is an uncontrollable action that results in a transition to $s_a$, then it means that the specification does not satisfy the controllability condition. 
The automaton $H$ in Fig. \ref{H1} shows an automaton which models the specification in which two consecutive lefts or rights ($\{ll, rr, lrr, rlrll,\hdots\}$) are deemed unsafe. 

\textit{Remark:} We have modeled the sequence of actions that are unsafe by ``marking'' some states as unsafe. The marked states have a completely different interpretation in supervisory control theory (SCT) where they are often associated with desirable features--- like signifying the end of an execution with the automata ``recognizing'' the string to be in its language. In fact, in SCT, if $\overline{L}_m(G) = L(G)$, then we say that $G$ is \textit{non-blocking}, which means that the system can always reach a target (marked) state from any state in its reachable set. Deadlock freedom is an instance of the non-blocking property. We will need some simple modifications in the method discussed above if we have to tackle both the interpretations simultaneously.

Algorithm \ref{Algo_1} shows the psuedocode for Supervised Q-learning for non-Markovian action and state constraints. It receives as input the automata $G_s'$ and $H_1$ describing the action constraints, the set of propositional symbols $\mathcal{P}$, the labelling function $L$, the discount factor $\gamma$, and the reward machines $\mathcal{R}$ over $\mathcal{P}$. The goal is to learn an optimal policy given the state and action constraints.

\begin{algorithm}
\caption{Supervised Q-learning for non-Markovian action and state constraints}
\label{Algo_1}
\begin{algorithmic}[1]
\STATE Input: ${G_s'}, H_1, \mathcal{P}, L, \gamma, \mathcal{R} = \langle U, \delta_u, \delta_r, u_0 \rangle$
\STATE Construct $G_s$ and $H$ from $G_s'$ and $H_1$ respectively
\STATE \label{init} $\widetilde{Q} \leftarrow$ InitializeQValueFunction()
\FOR{$l= 0$ to num\_episodes} \label{numep}
    \STATE {$u_j \leftarrow u_0; q_s \leftarrow q_{0s}; q_h \leftarrow q_{0h}$; $s \leftarrow$ EnvInitialState();}
    \FOR {$t=0$ to length\_episode}\label{lenep}
        \IF {EnvDeadEnd($s$)}
            \STATE {\textbf{break}}
        \ENDIF
    \STATE{$a \leftarrow$ GetActionEpsilonGreedy($\widetilde{q}_{q_s, q_h ,u_j}, s$) such that $(\delta_s(a, q_s) \neq s_a) \wedge (\delta_h(a, q_h) \neq s_a)$}\label{sup}
    \STATE $s' \leftarrow$ EnvExecuteAction$(s, a)$ \label{exe}
    \STATE $u_k \leftarrow \delta_u(u_j, L(s'))$  \label{RMtrans}
    \STATE $r \leftarrow \delta_r(u_j, u_k)$ \label{RMreward}
    \STATE $q_m \leftarrow \delta_s(q_s, a)$; $q_n \leftarrow \delta_h(q_h, a)$
    \IF {EnvDeadEnd($s'$)} \label{UpStart}
        \STATE {$\widetilde{q}_{q_s, q_h ,u_j}(s, a) = \widetilde{q}_{q_s, q_h ,u_j}(s, a) + {\alpha} r(s, a, s')$}
    \ELSE 
    \STATE{$\widetilde{q}_{q_s, q_h ,u_j}(s, a) = \widetilde{q}_{q_s, q_h ,u_j}(s, a)$ $+$ ${\alpha}$ $[r(s, a, s')$ $+$ $\gamma \max_{a'}\widetilde{q}_{q_m, q_n, u_k}(s', a') - \widetilde{q}_{q_s, q_h ,u_j}(s, a)]$}
    \ENDIF \label{UpEnd}
    \STATE $q_s \leftarrow \delta_s(q_s, a)$; $q_h \leftarrow \delta_h(q_h, a)$ 
    \STATE $u_j \leftarrow \delta_u(u_j, L(s'))$; $s \leftarrow s'$
        \ENDFOR
\ENDFOR
\end{algorithmic}
\end{algorithm}

The algorithm learns one q-value function per state of the automata and reward machine. That is, it will learn $|Q_s|\times|Q_h| \times |U|$ many q-value functions in total, where $|\cdot|$ denotes the size of the set argument. These q-functions are stored in $\widetilde{Q}$, and $\widetilde{q}_{q_s, q_h, u_j} \in \widetilde{Q}$ corresponds to the q-value function for states $q_s \in Q_s$ and $q_h \in Q_h$ of the two automata, and $u_j \in U$ of the reward machine.

After initializing $\widetilde{Q}$ in step \ref{init}, the algorithm has 2 nested loops: one for the number of episodes that we are running (step \ref{numep}), and the other for the length of each episode (step \ref{lenep}). Step \ref{sup} of the algorithm selects an action $a$ under supervision wherein only those actions that do not result in the next state $s_a$ in either of the two automata are considered for selection by epsilon-greedy method. Thereafter, the agent executes the action (Step \ref{exe}), evaluates the reward function through the transition in the reward machine (Steps \ref{RMtrans} and \ref{RMreward}) and then updates the q-value function (Steps \ref{UpStart}-\ref{UpEnd}) using the standard q-learning rule. Note that the maximization step is over $\widetilde{q}_{q_m, q_n, u_k}$ since those q-values would be used for the selection of $a'$ as $G_s, H_1$ and $\mathcal{R}$ would have transitioned. Action $a'$ is to be picked so that $s_a$ is not reached in $G_s$ and $H_1$.


\section{Illustrative Example}\label{example}
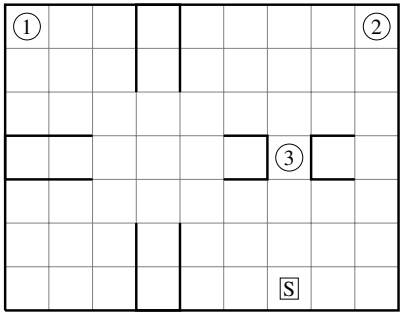
\begin{figure}[htbp]
    \centering
    \resizebox{.3\textwidth}{!}{
    \begin{tikzpicture}
    \draw[gray, very thin] (0, 0) grid (9, 7);
    
    \draw [ultra thick] (0,0) -- (9,0) -- (9,7) -- (0,7) -- (0,0);
    \draw [ultra thick] (3, 2) -- (3,0) -- (4,0) -- (4,2);
    \draw [ultra thick] (2, 3) -- (0,3) -- (0,4) -- (2, 4);
    \draw [ultra thick] (3, 5) -- (3,7) -- (4,7) -- (4,5);
    
    \draw [ultra thick] (5, 3) -- (6, 3) -- (6, 4) -- (5, 4);
    \draw [ultra thick] (8, 3) -- (7, 3) -- (7, 4) -- (8, 4);
    \node[shape=circle,draw,inner sep=2pt] at (0.5, 6.5) () {\Large 1};
    \node[shape=rectangle,draw,inner sep=2pt] at (6.5, 0.5) () {\Large S};
    \node[shape=circle,draw,inner sep=2pt] at (8.5, 6.5) () {\Large 2};

    \node[shape=circle,draw,inner sep=2pt] at (6.5, 3.5) () {\Large 3};

    \end{tikzpicture}}
    \caption{A pick-up and delivery grid world}
    \label{fig:my_label}
\end{figure}
\noindent \textbf{Setting:} Consider the pick-up and delivery grid world in Fig. \ref{fig:my_label}. The agent starts from the position $S$, picks up items labeled $1, 2$ and $3$, and returns back to $S$. The agent unloads the items in a last-in-first-out order and the unloading sequence must be $3-2-1$. Therefore, the nominal pick-up of the items is in sequence from $1$ to $3$. The objective of the agent is to finish this process in minimum number of steps.

\noindent \textbf{State and action space:} The state of the agent is composed of its position in the grid, its orientation and an indicator of the items that the agent has picked up. The position of the agent in the grid can be denoted by the coordinate of the respective pixel it occupies, with the bottom-leftmost corner of the grid denoting the origin. We can use a set of symbols $\{\uparrow, \downarrow, \leftarrow, \rightarrow\}$ to describe its orientation and a 3-bit binary variable to denote the items that the agent has picked up. The agent's movement in the grid is denoted by four actions $l, r, f$ and $b$ denoting the movement in four directions: left, right, forward and backward respectively. The actions $r$ and $l$ change the orientation of the agent by $90$ degrees clockwise and anticlockwise respectively. In addition, they also lead to a change in position by one pixel forward in the direction of the new orientation. For example, if $l$ are $r$ are feasible actions from a position $(x, y, \uparrow)$,then:
\begin{align}
    (x, y, \uparrow) \xrightarrow{l} (x-1, y, \leftarrow), \nonumber \\
    (x, y, \uparrow) \xrightarrow{r} (x+1, y, \rightarrow). \nonumber
\end{align}
The actions $f$ and $b$ do not change the orientation and only result in a change of position by one pixel forward or backward respectively relative to the current orientation of the agent. If $f$ are $b$ are feasible actions from a position $(x, y, \uparrow)$, then we have:
\begin{align}
    (x, y, \uparrow) \xrightarrow{f} (x, y+1, \uparrow), \nonumber \\
    (x, y, \uparrow) \xrightarrow{b} (x, y-1, \uparrow). \nonumber
\end{align}
We use $d_1$, $d_2$ and $d_3$ ($p_1, p_2$ and $p_3$) for actions denoting dropping (respectively picking) of the respective items by the agent. We assume that the agent cannot reliably carry all three items together because of which it can accidentally drop an item when it is carrying all three of them. For simplicity, we assume that it can accidentally drop only item 1. We model this by associating an uncontrollable action $d_u$. The illustration can easily be extended to account for dropping of either of the items.  

\noindent \textbf{Action Constraints} 
\begin{enumerate}
    \item The agent cannot make a $U$-turn. We interpret it has two consecutive rights or lefts. That is, the set $\{ll, rr\}$ denotes the illegal substrings of the string of actions. The automaton $H$, with initial state $0$, in Fig. \ref{H1} describes the illegal sequence of actions corresponding to the set of strings $\{ll, rr\}$. For ease of exposition, the actions $d_i$ and $p_i$ are not shown in $H$ and they can be added with a self loop around every state. 
    
    \item In order for the agent to satisfy the constraint on the unloading sequence, if an agent accidentally drops the first item then it must also immediately drop the other two and pick all three in sequence again. That is, $d_u$ must be followed by the string $d_2d_3p_1p_2p_3$. 
    Fig. \ref{F} shows the automaton $G_s$ that constrains the agent to take the action sequence $d_2d_3p_1p_2p_3$ when an uncontrollable action $d_u$ happens\footnote{For the case when the agent can drop either of the other two items, we can introduce two more uncontrollable actions $d_{u2}$ and $d_{u3}$ with remedial strings $d_1d_3p_1p_2p_3$ and $d_1d_2p_1p_2p_3$ respectively.}. 
\end{enumerate}

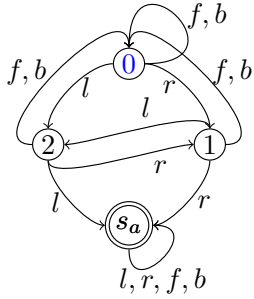
\begin{figure}[htbp]
    \centering
    \resizebox{0.25\textwidth}{!}{
    \begin{tikzpicture}
         \node[shape=circle,draw,inner sep=1pt](0) at (0,2) {\textcolor{blue}{$0$}};
         \node[shape=circle,draw,inner sep=1pt](1) at (1,1) {$1$};
         \node[shape=circle,draw,inner sep=1pt](2) at (-1,1) {$2$};
         \node[shape=circle,draw,inner sep=1pt]() at (0,0) {$s_a$};
         \node[shape=circle,draw,inner sep=2pt](sa) at (0,0) {$s_a$};
         
         \draw[->] (0) .. controls +(right:15mm) and +(up:15mm) .. node[right]{$f, b$}(0);
        \draw[->] (sa) .. controls +(down:10mm) and +(right:10mm) .. node[below]{$l, r, f, b$}(sa);
         \draw[->] (0) .. controls +(right:5mm) and +(up:5mm) .. node[left]{$r$}(1);
         \draw[->] (0).. controls +(left:5mm) and +(up:5mm) .. node[right]{$l$}(2);
         \draw[->] (1) .. controls +(down:5mm) and +(right:5mm) .. node[right]{$r$}(sa);
         \draw[->] (2).. controls +(down:5mm) and +(left:5mm) .. node[left]{$l$}(sa);
         
         \draw[->] (1).. controls +(up:5mm) and +(right:5mm) .. node[above]{$l$}(2);
         \draw[->] (2).. controls +(down:5mm) and +(left:5mm) .. node[below, near end]{$r$}(1);
         
         \draw[->] (2).. controls +(left:10mm) and +(up:10mm).. node[left]{$f,b$}(0);
         \draw[->] (1).. controls +(right:10mm) and +(up:10mm).. node[right]{$f,b$}(0);
         
     \end{tikzpicture}}
    \caption{An automaton $H$ with initial state $0$ that accepts strings equivalent to a U-turn by the agent; for example the strings \{$ll, rr, lrfbrr, \hdots$\}.}
    \label{H1}
\end{figure}
Fig. \ref{RM} shows the reward machine for imposing the specification of sequential pickup of items in which $\mathcal{P} = \{0, 1, \fbox{S}\}$. The labeling function generates a 3-bit binary number corresponding to the items that the agent has picked up. It uses the symbol \fbox{S} to indicate whether the agent reached location \fbox{S} in the grid. The RM starts at state $u_0$. It transitions to state $u_1$ if the items are not picked up in the desired sequence and thereafter obtains a reward of $-20$ at every step. The reward for every step increases from $-10$ to $-8$ (from $-8$ to $-6$) after the agent picks up item $1$ (respectively, item $2$). It further increases to $-1$ after the agent picks up all three items. Note that the RM only handles the sequential pickup initially. The uncontrollable drop of item $1$ after the agent picks all three items is handled by the automaton $G_s$. Once all three items are picked up in sequence, the RM stays at $u_4$ until the agent reaches \fbox{S}.\\

\noindent \textbf{Simulation Results}
The simulation was run for 500,000 iterations, where each iteration consisted of one episode of a maximum of 60 steps. The hyperparameters were set as $\alpha=0.1, \epsilon=0.25$ and $ \gamma=0.9$. Additionally, $\epsilon$ was decayed by 5\% after every period of 10,000 iterations. The single-step reward is defined as in Fig. \ref{RM} in accordance to the items picked and their order. Fig. \ref{score} shows the average score, i.e., the average accumulated reward over an episode, as training progressed. Fig. \ref{taskcompletion} shows the average percentage of times when the task was completed. All data shown are averages over intervals of 10,000 iterations during training. At the end of training, $\epsilon$ achieves a value of $0.019$. This $\epsilon$-greedy policy completes the task around 95\% of the time as shown in Fig. \ref{taskcompletion}. Upon inference, i.e., upon setting $\epsilon=0$ post training, the agent completes the task optimally 100\% of the time.
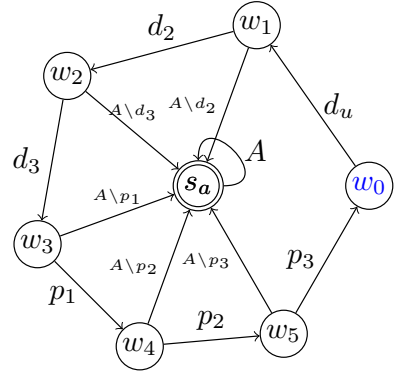
\begin{figure}
    \centering
    \resizebox{0.3\textwidth}{!}{
    \begin{tikzpicture}
    \node[shape=circle,draw,inner sep=1pt](sa) at (0,0) {$s_a$};
    \node[shape=circle,draw,inner sep=2pt](sa) at (0,0) {$s_a$};
    \path (sa) ++(70:2) node (1) [shape=circle,draw,inner sep=1pt] {$w_1$};
    \path (sa) ++(140:2) node (2) [shape=circle,draw,inner sep=1pt] {$w_2$};
    \path (sa) ++(200:2) node (3) [shape=circle,draw,inner sep=1pt] {$w_3$};
    \path (sa) ++(250:2) node (4) [shape=circle,draw,inner sep=1pt] {$w_4$};
    \path (sa) ++(300:2) node (5) [shape=circle,draw,inner sep=1pt] {$w_5$};
    \path (sa) ++(360:2) node (0) [shape=circle,draw,inner sep=1pt] {\textcolor{blue}{$w_0$}};

    \draw[->] (0)-- node[right]{$d_u$}(1);
    \draw[->] (1)-- node[above]{$d_2$}(2);
    \draw[->] (2)-- node[left]{$d_3$}(3);
    \draw[->] (3)-- node[left]{$p_1$}(4);
    \draw[->] (4)-- node[above]{$p_2$}(5);
    \draw[->] (5)-- node[left]{$p_3$}(0);

    \draw[->] (1)-- node[left]{\tiny $A\backslash d_2$}(sa);
    \draw[->] (2)-- node[above]{\tiny $A\backslash d_3$}(sa);
    \draw[->] (3)-- node[above]{\tiny $A\backslash p_1$}(sa);
    \draw[->] (4)-- node[left]{\tiny $A\backslash p_2$}(sa);
    \draw[->] (5)-- node[left]{\tiny $A\backslash p_3$}(sa);
    \draw[->] (sa) .. controls +(right:10mm) and +(up:10mm) .. node[right]{$A$}(sa);
\end{tikzpicture}}
    \caption{An automaton $G_s$ with initial state $w_0$ describing the sequence of safe actions when an uncontrollable action $d_u$ happens.}
    \label{F}
\end{figure}

\begin{figure}
    \centering
    \resizebox{0.45\textwidth}{!}{
    \begin{tikzpicture}
    \path (sa) ++(0:2) node (0) [shape=circle, draw, inner sep=1pt] {\textcolor{blue}{$u_0$}};
    \path (sa) ++(70:2) node (1) [shape=circle,draw,inner sep=1pt] {$u_1$};
    \path (sa) ++(140:2) node (2) [shape=circle,draw,inner sep=1pt] {$u_2$};
    \path (sa) ++(210:2) node (3) [shape=circle,draw,inner sep=1pt] {$u_3$};
    \path (sa) ++(280:2) node (4) [shape=circle,draw,inner sep=1pt] {$u_4$};
    \path (sa) ++(330:4) node (5) [shape=circle,draw,inner sep=1pt] {{$u_5$}};
    
    \draw[->] (0).. controls +(right:15mm) and +(down:15mm) .. node[right, near start]{$\langle 000, -10 \rangle$}(0);
    \draw[->] (1).. controls +(up:15mm) and +(right:15mm) .. node[right, near end]{$\langle true, -20 \rangle$}(1);
    \draw[->] (2).. controls +(up:15mm) and +(left:15mm) .. node[left, near end]{$\langle 100, -8 \rangle$}(2);
    \draw[->] (3).. controls +(left:15mm) and +(down:15mm) .. node[left, near end]{$\langle 110, -6 \rangle$}(3);
    \draw[->] (4).. controls +(left:15mm) and +(down:15mm) .. node[left, near end]{$\langle 111 \wedge (\neg \fbox{S}), -1 \rangle$}(4);

    \draw[->] (0)-- node[right]{$\langle 010 \vee 001, -20 \rangle$}(1);
    \draw[->] (0)-- node[above]{$\langle 100, -8 \rangle$}(2);
    \draw[->] (2)-- node[above]{$\langle 101, -20 \rangle$}(1);
    \draw[->] (2)-- node[left]{$\langle 110, -6 \rangle$}(3);
    \draw[->] (3)-- node[above, near end]{$\langle 111, -1 \rangle$}(4);
    \draw[->] (4)-- node[below]{$\langle 111 \wedge \fbox{S}, 0 \rangle$}(5);
    \draw[->] (5).. controls +(right:15mm) and +(down:15mm) .. node[below, near end]{$\langle true, 0 \rangle$}(5);
\end{tikzpicture}}
    \caption{Reward Machine for pick up and delivery grid world example.}
    \label{RM}
\end{figure}
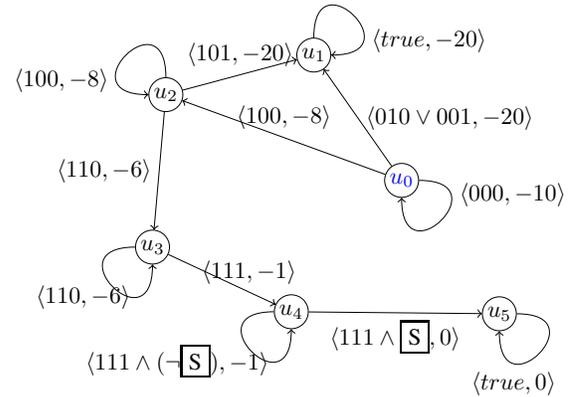




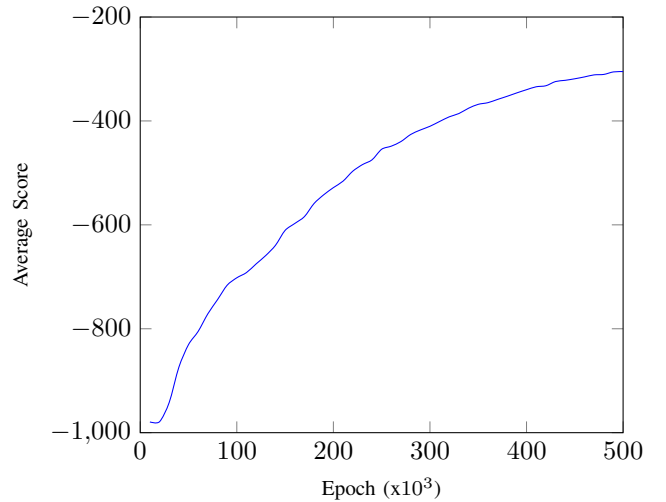
\begin{figure}
    \centering
    \begin{tikzpicture}
    \begin{axis}[
    xmin = 0, xmax = 500,
    ymin = -1000, ymax = -200,
    width = 0.45\textwidth,
    height = 0.4\textwidth,
    xlabel = {\footnotesize Epoch (x$10^3$)},
    ylabel = {\footnotesize Average Score},
    ]
    \addplot[smooth, blue] file[skip first] {score_trim_latex.dat};
    \end{axis}
    \end{tikzpicture}
    
    \caption{Average score achieved averaged across intervals of 10,000 iterations during training}
    \label{score}
\end{figure}

\begin{figure}
    \centering
    \begin{tikzpicture}
    \begin{axis}[
    xmin = 0, xmax = 500,
    ymin = 0, ymax = 100,
    width = 0.45\textwidth,
    height = 0.4\textwidth,
    xlabel = {\footnotesize Epoch (x$10^3$)},
    ylabel = {\footnotesize Average Task Completion \%},
    ]
    \addplot[smooth, blue] file[skip first] {tc_trim_latex.dat};
    \end{axis}
    \end{tikzpicture}
    
    \caption{Percentage of times task was completed averaged across intervals of 10,000 iterations during training}
    \label{taskcompletion}
\end{figure}
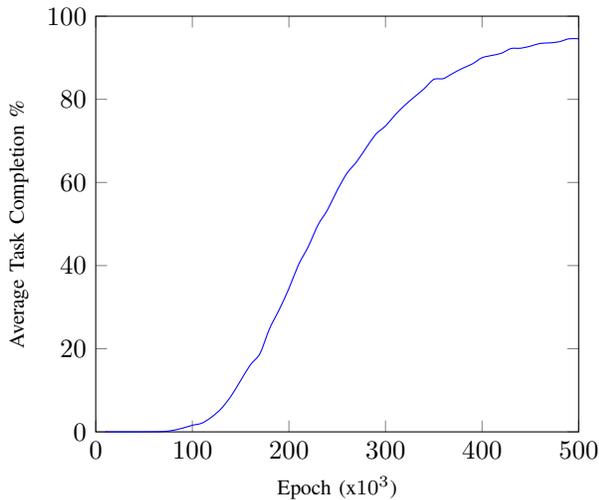


\section{Conclusion}\label{concl}
In this paper, we borrowed concepts from supervisory control theory of discrete event system to develop a method to learn optimal control policies for a finite-state Markov Decision Process in which (only) certain sequences of actions are deemed unsafe (respectively safe) while accounting for the possibility that a subset of actions might be uncontrollable. We assumed that the constraints can be modeled using a finite automaton and a natural future direction of research is to develop such methods for a more general class of constraints.

\bibliographystyle{IEEEtran}
\bibliography{ref}
\end{document}